\pdfoutput=1
\documentclass[sigconf]{aamas} %
\usepackage{balance}

\makeatletter
\gdef\@copyrightpermission{
  \begin{minipage}{0.2\columnwidth}
   \href{https://creativecommons.org/licenses/by/4.0/}{\includegraphics[width=0.90\textwidth]{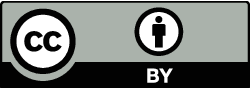}}
  \end{minipage}\hfill
  \begin{minipage}{0.8\columnwidth}
   \href{https://creativecommons.org/licenses/by/4.0/}{This work is licensed under a Creative Commons Attribution International 4.0 License.}
  \end{minipage}
  \vspace{5pt}
}
\makeatother

\setcopyright{ifaamas}
\acmConference[AAMAS '25]{Proc.\@ of the 24th International Conference
on Autonomous Agents and Multiagent Systems (AAMAS 2025)}{May 19 -- 23, 2025}
{Detroit, Michigan, USA}{Y.~Vorobeychik, S.~Das, A.~Nowé  (eds.)}
\copyrightyear{2025}
\acmYear{2025}
\acmDOI{}
\acmPrice{}
\acmISBN{}

\acmSubmissionID{bs034}


\usepackage{graphicx}
\usepackage{subfigure}
\usepackage{algorithm}
\usepackage{algpseudocode}
\usepackage{caption}
\usepackage{enumitem}

\usepackage{xparse}
\usepackage{xargs}
\usepackage{amsmath}
\usepackage{amsthm}

\usepackage{pgfplots}
\pgfplotsset{compat=1.17}

\NewDocumentCommand{\Reals}{}{\mathbb{R}}
\NewDocumentCommand{\Nats}{}{\mathbb{N}}

\theoremstyle{plain}
\newtheorem{theorem}{Theorem}[section]

\theoremstyle{definition}
\newtheorem{definition}[theorem]{Definition}

\theoremstyle{remark}
\newtheorem*{remark*}{Remark}
\newtheorem*{setup*}{Setting}

\usepackage[disable,textsize=tiny]{todonotes}
\newcommandx{\rephrase}[2][1=]{\todo[linecolor=black,backgroundcolor=black!25,bordercolor=black,#1]{OR: #2}}

\newcommandx{\ltt}[2][1=]{\todo[linecolor=blue,backgroundcolor=blue!25,bordercolor=blue,#1]{\textbf{LTT:} #2}}

\title[Market-based AI]{Market-based Architectures in RL and Beyond}
\subtitle{Blue Sky Ideas Track}

\author{Abhimanyu Pallavi Sudhir}
\affiliation{
  \institution{University of Warwick, UK}
  \city{}
  \country{}
  }
\email{abhimanyu.pallavi-sudhir@warwick.ac.uk}

\author{Long Tran-Thanh}
\affiliation{
  \institution{University of Warwick, UK}
  \city{}
  \country{}
  }
\email{long.tran-thanh@warwick.ac.uk}

\begin{abstract}

Market-based agents refer to reinforcement learning agents which determine their actions based on an internal market of sub-agents. We introduce a new type of market-based algorithm where the state itself is factored into several axes called ``goods'', which allows for greater specialization and parallelism than existing market-based RL algorithms. Furthermore, we argue that market-based algorithms have the potential to address many current challenges in AI, such as \emph{search}\todo{idk if it sounds weird to say this}, \emph{dynamic scaling} and \emph{complete feedback}, and demonstrate that they may be seen to generalize neural networks; finally, we list some novel ways that market algorithms may be applied in conjunction with Large Language Models for immediate practical applicability.
    
\end{abstract}


\begin{document}

\pagestyle{fancy}
\fancyhead{}
\maketitle

\section{Introduction}

Before neural networks won the mandate of heaven, an AI research paradigm that had shown considerable promise was that of \emph{market-based architectures}, i.e. AI agents that determine their output based on some internal market-based mechanism \cite{kweeMarketBasedReinforcementLearning2001, changDecentralizedReinforcementLearning2020}.
\ltt{refs}

There are general intuitive arguments that motivate such a line of research. Philosophers and psychologists have long pondered multi-agent models of the mind \cite{minskySocietyMind1988}; moreover, one may imagine that ``any'' machine learning task could in principle be solved by a market of agents solving sub-tasks with their individual reward set by the sale value of their output. There is work suggesting that markets can capture some notion of \emph{bounded rationality}, e.g. the \emph{Boundedly Rational Inductive Agent} (BRIA) \cite{oesterheldTheoryBoundedInductive2023} and \emph{Algorithmic Bayesian Epistemology} \cite{neymanAlgorithmicBayesianEpistemology2024}. In some sense, markets ``aggregate'' the intelligence or capacities of their individual participants.\footnote{See e.g. Hayek on the role of markets in aggregating information \cite{hayekEconomicsKnowledge1937} -- or the famous parable ``I, Pencil'' \cite{leonardereadPencilMyFamily1958}: ``... no one person, no matter how smart, could create from scratch a small, everyday pencil [yet the market makes over a billion of them each year] ...''. There is also some empirical work on emergent intelligent behaviour in markets comprised of zero-intelligence traders \cite{godeAllocativeEfficiencyMarkets1993, schwartzHowMuchIrrationality2008, jamalSimpleAgentsIntelligent2015}}\todo{see commented-out stuff for potential extra stuff to add}

``Market-based architectures'' can be made concrete in the case of reinforcement learning (RL), where the majority of work in this area lies (see \ref{sec:related} for a brief summary). For example in the \emph{Hayek machine} \cite{baumModelIntelligenceEconomy1999} and its derivatives, the setting is a Markov Decision Problem (MDP) and there is a single resource, the ``right to act and collect reward'', that is traded between sub-agents. At each time step, this resource is sold to the highest-bidding sub-agent, who performs some action that modifies the state and collects reward.%

In this paper, we argue that market-based agents represent an underexplored and promising niche, \emph{especially} in context of recent advancements in language models (LLMs), and have potential to address a range of present challenges in contemporary AI research. Specifically, we make the following claims and contributions:

\textbf{Theoretical framework for market-based agents.} We present two general frameworks for market-based RL agents: (1) the ``deep market'' (Def~\ref{def:deep}), where a single good, the \emph{state}, is passed through a sequence of transacting agents, and (2) the ``wide market'' (Def~\ref{def:goods}), in which the state space itself is partitioned into factors called \emph{goods}. The deep framework is not much of a departure from existing algorithms, and can be applied to any Partially Observed Markov Decision Process (POMDP); to our knowledge the wide framework is original to us, and is a generalization of the deep framework which better mirrors the success of real-world markets allowing for greater specialization and parallelism. A Python library for creating market-based algorithms will be released upon publication.

\textbf{Markets, neural networks and backpropagation.} We demonstrate that these market-based agents can in principle be applied even to basic supervised learning tasks such as classification, and that neural networks (though not backpropagation or gradient descent) emerge as a special case of them. Furthermore we generalize the result in \cite{wentworthCompetitiveMarketsDistributed2018} to wide markets, demonstrating a suggestive relationship between backpropagation and markets at equilibrium.

\textbf{Search, complete feedback and alignment.} We claim that markets can address several present problems in AI research, specifically: they are a natural framework for \emph{search}, their scale or depth can be \emph{dynamic} rather than fixed, and they allow \emph{complete feedback} \cite{demskiCompleteFeedback2024}, a property widely regarded as valuable in AI alignment.

\textbf{Markets and LLMs.} We present novel ways in which market algorithms might be applied in conjunction with LLMs to address their limitations: they can be used for developing reasoning models like \texttt{o1}, and LLMs can facilitate ``information markets'' that can in turn improve human feedback mechanisms in AI training.

\subsection{Related Work}
\label{sec:related}

\textbf{Market-based RL.} The majority of early work in this area has focused on market-based \emph{reinforcement learning} (RL) algorithms. The pioneering work in this domain consists of Holland's \emph{Learning Classifier Systems} or ``bucket brigade'' \cite{hollandPropertiesBucketBrigade1985a} in rule-based systems, where condition-action agents (``classifiers'') bid to post messages (which can be actions described in some language) onto a global message board. Improvements to this paradigm were made by Schmidhuber \cite{schmidhuberEvolutionaryPrinciplesSelfreferential1987, schmidhuberLocalLearningAlgorithm1989} who allowed agents to determine their own bids and imposed credit conservation, and by Baum \cite{baumModelIntelligenceEconomy1999} who further strictly enforced property rights, resulting in a much more familiar set-up called the ``Hayek Machine''. The Hayek machine, which can be applied to any Markov Decision Process (MDP), was subsequently extended to POMDPs by \cite{kweeMarketBasedReinforcementLearning2001} by adding external memory, and more recently \cite{changDecentralizedReinforcementLearning2020} modified the framework to use Vickrey auctions and prove that a Nash equilibrium of the market produces a globally optimal policy.\todo{is this much description necessary?}

\textbf{Bounded rationality and markets.} Other, more recent work in this area includes: \emph{Logical Induction} \cite{garrabrantLogicalInduction2020}, an algorithm that assigns probabilities to mathematical sentences based on their prices in a prediction market that (roughly speaking) pays off when a sentence is proven; and the \emph{Boundedly Rational Inductive Agent} (BRIA) \cite{oesterheldTheoryBoundedInductive2023} which solves finite decision problems by assigning it to the highest-bidding trader, similar to in market-based RL. The key insight of these works is that markets are useful for modeling \emph{boundedly rational agents}. The main results of each work -- the fact that the logical inductor cannot by dominated by any polynomial-time trader, and the ``boundedly rational inductive agent criterion'' in the latter paper -- are specific and precise formulations of the Efficient Market Hypothesis \cite{neymanAlgorithmicBayesianEpistemology2024}.\todo{remove last sentence?}

\textbf{Markets and neural networks.} A specific equivalence between classifier systems and neural networks has been studied in the \emph{Neural bucket brigade} \cite{davisMappingClassifierSystems1988, schmidhuberLocalLearningAlgorithm1989}, although this does not consider backpropagation. A suggestive analogy between markets and backpropagation is discussed in \cite{wentworthCompetitiveMarketsDistributed2018}, though only for the case of a strictly sequential, unit-width market like that in Def~\ref{def:deep}. In our work we make this more precise and generalize it to \ref{def:goods}\todo{maybe mention Agoric Computation, BitTensor and market-based scaffolding for LLMs? see comments}

\section{Market algorithms}

\NewDocumentCommand{\statevecs}{}{\mathcal{S}}
\NewDocumentCommand{\statevec}{}{s}
\NewDocumentCommand{\actions}{}{\mathcal{X}}
\NewDocumentCommand{\action}{}{x}
\NewDocumentCommand{\trans}{}{\mathbf{P}}
\NewDocumentCommand{\reward}{}{\mathbf{R}}
\NewDocumentCommand{\obss}{}{\Omega}
\NewDocumentCommand{\obs}{}{\omega}
\NewDocumentCommand{\obsd}{}{\mathbf{O}}
\NewDocumentCommand{\agents}{}{\mathcal{A}}
\NewDocumentCommand{\agent}{}{\alpha}
\NewDocumentCommand{\agentact}{}{\hat{\agent}}
\NewDocumentCommand{\agentbid}{}{{\agent}_b}
\NewDocumentCommand{\consumer}{}{\gamma}
\NewDocumentCommand{\consumers}{}{\mathcal{C}}
\NewDocumentCommand{\consumeract}{}{\hat{\consumer}}
\NewDocumentCommand{\consumerbid}{}{{\consumer}_b}
\NewDocumentCommand{\messages}{}{\mathrm{str}}
\NewDocumentCommand{\wealth}{}{w}
\NewDocumentCommand{\bids}{}{b}
\algnewcommand{\LeftComment}[1]{\(\triangleright\) #1}

\begin{setup*}[POMDP]
We assume a ususal POMDP setting, with a state space $\statevecs$, action space $\actions$, transition probability $\trans(\statevec'\mid\statevec,\action)$ (which allows us to treat actions as stochastic functions i.e. $\action(\statevec)\sim\trans(\statevec'\mid\statevec,\action)$), reward function $\reward(\statevec,\action,\statevec')$, and observation distribution $\obs(\statevec)\sim\obsd(\obs\mid\statevec)$ over a set of observations $\obss$. A policy is a map $\agent:\obss\to\actions$, and the process proceeds as per usual.
\label{su:pomdp}
\end{setup*}

Mirroring \cite{kweeMarketBasedReinforcementLearning2001} and similar to Belief-MDP formulations, we can extend the state and message spaces by taking the cartesian product space with a message space $\messages$ which is always preserved by $\obs$; this gives the policy a ``memory'', or in terms of markets, creates informational goods. %

The first algorithm we describe is Def~\ref{def:deep}: here, agents bid at each time step for the \emph{right to act and collect reward}, and the highest-bidding agent is chosen to act. As in previous work, e.g. \cite{baumModelIntelligenceEconomy1999, changDecentralizedReinforcementLearning2020}, these agents are not utility-maximizers but programs out of a possibly infinite collection of agents $\agents$ (which we leave abstract). The parameters of this algorithm are the wealths of each agent $\wealth[\agent]$. trained by the training loop \textsc{Capitalism}: at each step, the agent pays its bid to the previous agent, collects the reward generated by its actions and receives the bid of the next. This means that at equilibrium, each agent is incentivized to bid the value function, and perform the action with maximum Q-value.\todo{Is this ok for an explanation? idk if I should remove the last sentence as I'm not explicitly proving it}

\begin{definition}[Deep market]
Assume a POMDP setup, and let $\agents$ be a collection of ``agents'', which are (stochastic) maps $\agent:\obss\to\actions\times\Reals$. The first component $\agentact:\obss\to\actions$ of an agent is called its \emph{action}, the second component $\agentbid:\obss\to\Reals$ is called its \emph{bid}. The market algorithm then proceeds as in Algorithm~\ref{alg:deep}.\ltt{briefly explain Alg 1 in layman's language}\footnote{For training, this may be executed in multiple episodes with different initial state $\statevec$, either with finite episodes or in parallel (with shared wealth variables across running instances).}

\begin{algorithm}[tb]
\caption{Deep market}
\label{alg:deep}
\begin{algorithmic}

\Procedure{Market}{}\Comment{Forward pass}
\State \textbf{parameters: } $\wealth [\agent]\in\Reals$ \Comment{wealths of each $\agent\in\agents$}
\State \textbf{input: $\obs\in\obss$}
\State $\bids[\agent]\gets\min(\agentbid(\obs),\wealth[\agent])$ for $\agent\in\agents$\Comment{Cap bids by wealth}
\State $\agent^*\gets\arg\max b[\alpha]$\Comment{Choose winning agent}
\State $\action\gets \agentact^*(\obs)$\Comment{Determine action}
\State \Return $\action$, $\agent^*$
\EndProcedure

\Procedure{Capitalism}{}\Comment{Training loop}
\State Initialize agent wealths $\wealth [\agent]\in\Reals$ for each $\agent\in\agents$
\State Initialize original owner of the world $\agent^*$
\State Initialize state $\statevec\in\statevecs$
\While{$t\in\Nats$}
\State $\obs\gets\obs(\statevec)$\Comment{Generate observation}
\State $\agent^*_{\text{prev}}\gets\agent^*$
\State $\action, \agent^*\gets \Call{Market}{\obs}$
\State $\wealth[\agent^*]\gets\wealth[\agent^*]-\bids[\agent^*]$\Comment{Pay bid}
\State $\wealth[\agent^*_{\text{prev}}]\gets\wealth[\agent^*_{\text{prev}}]+\bids[\agent^*]$\Comment{to previous owner}
\State $\statevec_{\text{prev}}\gets\statevec$
\State $\statevec\gets\action(\statevec)$\Comment{Transition state}
\State $\wealth[\agent^*]\gets\wealth[\agent^*]+\reward(\statevec_{\text{prev}},\action,\statevec)$\Comment{add reward to wealth}
\EndWhile
\EndProcedure
\end{algorithmic}
\end{algorithm}

\label{def:deep}
\end{definition}

Some details have been ignored. $\agents$ will usually be infinite and so tables like $\wealth[\agent]$ cannot simply be indexed on it: instead, $\agents$ must be countably enumerated and added to the economy one-by-one in the training loop with each agent being endowed with some allowance. To prevent holdout problems, one may impose a small fixed ``rent'' on $\agent^*$ at each training step i.e. $\wealth[\agent^*]\gets(1-\varepsilon)\wealth[\agent^*]$. A simple first-price auction is shown for simplicity, and may be replaced with a Vickrey auction in line with \cite{changDecentralizedReinforcementLearning2020}. One may also replace the explicit reward function $\reward$ with a class of ``consumers'' $\consumers$ who place bids upon desirable states, which may be a useful formulation for reinforcement learning from diverse human feedback\footnote{see \cite{conitzerPositionSocialChoice2024, geAxiomsAIAlignment2024a} for a primer on this area}.\todo{might instead move this to the markets and LLMs section}

Definition~\ref{def:deep}, which subsumes existing market-based RL, already illustrates one of the key defining features of markets recognizable to any student of economics: markets serve not only to \emph{select} (via market competition) the best process to achieve a task, but also to \emph{distribute} a complex task among agents which are individually much too weak or uninformed to complete the entire task. This means that the collection of actions $\agents$ can be a class of ``simple'' agents, so that enumerating $\agents$ can quickly find many valuable agents. %

Specifically, Def~\ref{def:deep} exploits modularity of \emph{action}, where the state can be transformed one step at a time. There is however another form of modularity, missed by all existing market-based RL algorithms, which we may call modularity of \emph{state}, and is crucial to the success of real-world markets:  here, agents are not constantly transacting the whole ``state of the world'': instead, the state of the world is decomposed into several components, called \emph{goods}\footnote{$\oplus$ denotes the direct sum of vector spaces, which is a Cartesian product equipped with a pointwise vector addition operator}: $\statevecs=\statevecs_1\oplus\dots\oplus\statevecs_n$. For instance, $\statevec_1\in\statevecs_1$ might represent the quantity of iron ore in the world. Agents bid for small quantities of each good; no agent owns the whole world, and does not have to bother performing a valuation of the whole world. The ``state of the world'' may be recovered as the vector sum of all agents' holdings.%

\NewDocumentCommand{\equilibrium}{}{\operatorname{Equ}}
\NewDocumentCommand{\price}{}{\mathbf{p}}

At least two new difficulties are introduced by considering markets of multiple divisible goods:

\textbf{General equilibrium theory.} Allocating goods is no longer as easy as an auction, because agents might have joint demand schedules for goods that are complementary or substitute to each other. The problem of matching buyers and sellers in this setting is the domain of ``General Equilibrium Theory'' in economics, where there are models such as the Fisher market and the Arrow-Debreu exchange market \cite{arrowExistenceEquilibriumCompetitive1954, mckenzieExistenceGeneralEquilibrium1959}. Computing the equilibrium in these models is non-trivial and often intractable \cite{chenSettlingComplexityArrowDebreu2009, chenSpendingNotEasier2009a}

\textbf{Property rights in POMDPs.} A more subtle difficulty lies in the fact that we want to divide the state $\statevec\in\statevecs$, which is not directly observed, among bidding agents (so that each of their actions only transform their respective portions of the state, i.e. their properties), but the agents only submit demand schedules over $\obs\in\obss$. It is not obvious how to map a decomposition of a vector $\obs(\statevec)$ back onto $\statevec$.

Both of these have to do with specific questions of how buyers and sellers meet and match in real markets, i.e. having to do with \emph{institutions} such as property rights and mechanism design. These questions are out of scope for us, and we abstract them away by postulating some effective equilibrium computation algorithm\footnote{e.g. there are results demonstrating that simple tâtonnement converges to a Walrasian equilibrium when the agents' valuations are gross subtitutes \cite{GUL199995}.} $\equilibrium(\obs,\agentbid^1,\dots\agentbid^m)=(\price, \obs[\agent^1],\dots\obs[\agent^m])$ i.e. which takes the total perceived quantity of goods in the world $\obss$ and each agent's valuation function $\agentbid^i:\obss\to\Reals$, and returns a price vector $\price\in\obss$ and allocations to each agent $\obs[\agent^i]\in\obss$, such that (in line with a Walrasian equilibrium with quasilinear utilities \cite{millerNotesMicroeconomicTheory2006}):
\begin{itemize}
    \item $\obs=\sum\obs[\agent^i]$ (the full quantity is allocated)
    \item $\price\cdot\obs[\agent^i]\le\agentbid^i(\obs[\agent^i])$ for all $\agent^i$ (no agent pays for what it doesn't value), and 
    \item $\obs[\agent^i]=\arg\max_{\obs'\in\obss}\agentbid^i(\obs')-\price\cdot\obs'$ for all $\agent^i$ (each agent gets a utility-maximizing bundle at the given price).
\end{itemize}

\begin{definition}[Wide market]
  Everything from the POMDP setup and the agent type in Def~\ref{def:deep} remains the same; except that $\statevecs$ and $\obss$ are now vector spaces with each vector called a \emph{goods bundle}. Further, we have action spaces $\actions_\obs$ indexed by $\obs\in\obss$ such that (1) for any $\action\in\actions_\obs$, there is an ``exercised property right'' denoted $\statevec_\action(\obs)\in\statevecs$ such that $\obs(\statevec_\action(\obs))=\obs$ and $\action(\statevec)=\action(\statevec_\action(\obs))+(\statevec-\statevec_\action(\obs))$ (i.e. each agent's actions transform only the goods they own) and (2) there is an injective map $\xi:\actions_{\obs_1}\times\actions_{\obs_2}\to\actions_{\obs_1+\obs_2}$ such that $\xi(\action_{\obs_1}, \action_{\obs_2})=\action_{\obs_1}(\statevec_{\action_{\obs_1}}(\obs_1))+\action_{\obs_2}(\statevec_{\action_{\obs_2}}(\obs_2))+(\statevec-\statevec_{\action_{\obs_1}}(\obs_1)-\statevec_{\action_{\obs_2}}(\obs_2))$ (this is used to combine actions by different agents). The agents now have dependent type signatures $\agent:(\obs:\obss)\to\actions_\obs\times\Reals$, and the transition probability $\action(\statevec)\sim\trans(\statevec'\mid\statevec,\action)$, reward function $\reward(\statevec,\action,\statevec')$ and observation distribution $\obsd(\obs\mid\statevec)$ are now interpreted as applying to ``private property'', i.e. to any goods bundle in their respective domains, rather than to the whole state, e.g. each action $\action(\statevec)$ defines a \emph{production function} that transforms one goods bundle into another, and $\agentbid$ is an agent's \emph{valuation function} over all possible bundles, i.e. how much it is willing to pay for a particular perceived bundle (if it's differentiable, then $\nabla\agentbid(\obs)$ can be interpreted as the price vector it offers). The market algorithm proceeds as in Algorithm~\ref{alg:goods} \ltt{also explain Alg 2}.

\begin{algorithm}[tb]
\caption{Wide market}
\label{alg:goods}
\begin{algorithmic}

\Procedure{Market}{}\Comment{Forward pass}
\State \textbf{parameters: } $\wealth [\agent]\in\Reals$ \Comment{wealths of each $\agent\in\agents$}
\State \textbf{input: $\obs\in\obss$}
\State\LeftComment{Cap bids by budget}
\State$\bids[\agent]\gets\lambda\statevec:\min(\agentbid(s),\wealth[\agent])$ for all $\agent\in\agents$
\State\LeftComment{Compute equilibrium prices and allocations}
\State $\price,\obs'[\agent^1],\dots\obs'[\agent^n]\gets\equilibrium(\sum\obs[\agent], \bids[\dots])$
\State $\action[\agent]\gets\agentact(\obs'[\agent])$ for all $\agent\in\agents$\Comment{Determine actions}
\State \Return $\action[\agent^1], \dots\action[\agent^n], \price,\obs'[\agent^1],\dots\obs'[\agent^n]$
\EndProcedure

\Procedure{Capitalism}{}\Comment{Training loop}
\State Initialize agent wealths $\wealth [\agent]\in\Reals$ for each $\agent\in\agents$
\State Initialize agent properties $\obs[\agent]\in\obss$ for each $\agent\in\agents$
\State Initialize state $\statevec\in\statevecs$
\While{$t\in\Nats$}
\State $\obs\gets\obs(\statevec)$\Comment{Generate observation}

\State $\dots\gets\Call{Market}{\obs}$\Comment{get all outputs}

\State $\wealth[\agent]\gets\wealth[\agent]-\price\cdot\obs'[\agent]$ for all $\agent$\Comment{Charge buyers}
\State $\wealth[\agent]\gets\wealth[\agent]+\price\cdot\obs[\agent]$ for all $\agent$\Comment{Pay sellers}

\State $\statevec[\agent]\gets\statevec_{\action[\agent]}(\obs[\agent])$ for all $\agent$\Comment{Calculate property rights}
\State $\statevec'[\agent]\gets\action[\agent](\statevec[\agent])$ for all $\agent$\Comment{Transform goods}
\State $\wealth[\agent]\gets\wealth[\agent]+\reward(\statevec[\agent], \action[\agent], \statevec'[\agent])$ for all $\agent$
\State $\statevec\gets\sum\statevec'[\agent]$ \Comment{Update state}
\EndWhile
\EndProcedure

\end{algorithmic}
\end{algorithm}

\label{def:goods}
\end{definition}
\todo{is the preceeding explanation not enough? Otherwise I can add: [see following comment in tex source]}
\textbf{Computing prices via backpropagation}. Though it remains to be seen how standard RL problems might be cast in this setting, we expect implementations of this algorithm to be much more effective than of Def~\ref{def:deep}, as it allows us to use simpler and more specialized agents in the collection $\agents$. In particular, these agents do not need to estimate the valuations of the whole world, but only of their particular input goods.

This last point can be illustrated particularly nicely when the setup is an MDP, and rewards are replaced by consumers -- here, $\obss=\statevecs$ and $\obs(\statevec)=\statevec$, so $\agentact:\statevecs\to\actions$ can directly be interpreted as a production function $\agentact:\statevecs\to\statevecs:=\agentact(\statevec)(\statevec)$. Then if the agent can estimate what the market prices of its output goods will be (e.g. if prices are sufficiently stable that it makes sense to speak of a ``prevailing price'' $\price$), then it can compute its offered prices via the chain rule -- where $D\agentact$ denotes the Jacobian:

\begin{equation}
    \nabla\agentbid=D\agentact\cdot\price
    \label{eq:backprop}
\end{equation}

i.e. once the market ``graph'' is fixed, prices can be computed by simply backpropagating consumer bids through the graph. This generalizes the result in \cite{wentworthCompetitiveMarketsDistributed2018}, which demonstrated this relationship for deep markets only.

\section{Motivation for market-based AI}

\NewDocumentCommand{\nnX}{}{\mathbb{X}}
\NewDocumentCommand{\nnY}{}{\mathbb{Y}}

In this section, we describe how markets could potentially generalize neural networks and provide a more ``flexible training mechanism''. Although we have presented our algorithms in an RL setting, they can even be applied to supervised learning tasks by treating internal representations as ``states''. To see this, it is illustrative to see how a simple neural network can be recast as a market.

\begin{theorem}[Neural networks as markets]
  Consider a fully-connected neural network $f:\nnX\to\nnY:=f_n\circ\dots f_1$ where each $f_i:\Reals^{m_{i-1}}\to\Reals^{m_i}$ is a layer, i.e. a function of the form $f_i(\mathbf{x})=\sigma(W_i\mathbf{x}+\mathbf{b}_i)$ where $\sigma$ is a ReLU activation. Then there is a deep market whose forward pass performs the same operation as $f$.
  \label{thm:nn}
\end{theorem}
\begin{proof}
  The construction is straightforward. Define the state space $\statevecs:=\left[\bigoplus_{1\le i\le n}\statevecs_i\right]\oplus\nnY$ (with each $\statevecs_i:=\Reals^{m_i}$), with $\obs:\statevecs\to\obss$ discarding only the last component $\nnY$ which represents the true label which is unchanged under all actions. Each $\actions_{\obs}=\{(W_i,\mathbf{b}_i):W_i\in\Reals^{m_i\times m_{i-1}}, \mathbf{b}_i\in\Reals^{m_i}\}$ if $\obs\in\statevecs_{i-1}$ and empty if no such $i$ exists, and an action $\action=(W_i,\mathbf{i})$ acts on $\statevec\in\statevecs_{i-i}$ as $\action(\statevec)=\sigma(W_i\statevec+\mathbf{b}_i)$. The reward $\reward(\statevec,\action,\statevec')=-\ell(\statevec',\statevec'_{\nnY})$ for some loss function $\ell$ if $\statevec'\in\statevecs_n$ and $0$ otherwise. Finally, let $\agents$ consist of all constant maps to $\actions_\obs$ and endow non-zero wealth to only those agents whose actions' parameters are the same as some $f_i$.
\end{proof}

While the market model, i.e. the forward pass, in Theorem~\ref{thm:nn} is the same as the neural network, the training mechanism is \Call{Capitalism}{} (as defined in Algorithm~\ref{alg:deep}) rather than backpropagation. Detailed below are some strengths of this we anticipate:

\textbf{Search and dynamic scale.} \emph{Reasoning} is widely touted as a key limitation of current-day LLMs \cite{huangReasoningLargeLanguage2023,mclauAISearchBitterer2024}. A view held by some researchers including Yann LeCun \cite{yannlecunMachinesThatCan2023}\todo{is this ok to do?}, is that this is due to the fact that ``[neural networks] produce their answers with a constant number of computational steps between input and output'', independent of the complexity required by the problem. Some proposed architectures that avoid this limitation include dynamic neural networks \cite{hanDynamicNeuralNetworks2021}, adaptive computation time \cite{gravesAdaptiveComputationTime2017} as well as chain-of-thought based methods such as \texttt{o1} \cite{openaiLearningReasonLLMs2024}. Markets provide a principled alternative, as here the structure of the computational graph is itself learned, and different agents and structures may be active for different inputs.

\textbf{Complete feedback.} Informally speaking, markets allow \emph{any} aspect of the system to be optimized. Formal results are needed to make this statement precise, but intuitively: any aspect of a learner, such as any hyperparameter, or meta-learning, can be changed by adding a trader to the market who will profit if his changes are beneficial and the incentives are correctly designed. This is suggestive of the notion of ``complete feedback'' in AI alignment research, which refers to the property that ``the trainer can enact \emph{any} modification they'd like to make to the system'' \cite{demskiCompleteFeedback2024}, and is viewed as a desirable characteristic of an AI system for alignment. %

\section{Practicality and future work}\todo{is a merged conclusion+last section like this fine?}

We have presented two general frameworks for market-based RL agents, and illustrated that they may be seen to generalize neural networks in a supervised learning setting, albeit with a more flexible training mechanism that holds promise to address the limitations of current-day AIs with respect to reasoning and alignment properties.\todo{too overconfident?}

Despite these theoretical strengths, our algorithm as described faces practical challenges to implement in real-world machine learning tasks: blindly enumerating large classes of even simple agents is inefficient (compared to backpropagation, where the search is guided by gradients), and we have to store many more agents in memory than the ``size'' of the network (the exact number depending on the rule we use to prune low-wealth agents). Some potentially promising approaches include:
\begin{itemize}
\item ``integrated'' models which perform backpropagation by default but intelligently resort to markets when it expects changing the network structure to be worthwhile
\item having each agent simultaneously learn its parameters via backpropagation
\item decentralized set-ups, perhaps using frameworks such as BitTensor \cite{raoBitTensorPeertoPeerIntelligence2021}\todo{is it ok to mention bittensor?}, allowing traders to be shared across machine learning applications.
\end{itemize}

\textbf{Markets of LLMs.} A more immediately feasible application is to develop \emph{markets comprised of LLMs}, i.e. where $\agents$ is a collection of LLM agents. For instance, one may let $\statevecs=\obss$ be a message space, and let actions act on $\statevec$ by appending some ``chain-of-thought item'' to the current message. The final reward is determined by human feedback, and intermediate rewards by bids. Such a market would function as a ``reasoning model'' analogous to \texttt{o1}.\todo{should I write this formally as a Definition?}

The extension to a wide market is also immediate: agents may bid for the right to read only a portion of the message space\footnote{As for how to enable the agent to ``inspect'' the message to make an informed bid without it stealing the entire message, \cite{rahamanLanguageModelsCan2024} is relevant: the agent can subcontract another LLM to inspect the message and place the bid, then have its context deleted.} -- this allows for more precise credit assignment to contributions by different agents, and may be understood as to \emph{trees-of-thought} \cite{yaoTreeThoughtsDeliberate2024} what \texttt{o1} is to \emph{chain-of-thought}.

\textbf{Theoretical work.} The most pressing need at present is for \emph{precise theoretical results} on the effectiveness of market-based algorithms. An immediate research agenda includes the following:

\begin{itemize}
  \item Determining \textbf{convergence and optimality conditions} of market algorithms; in particular, generalizing the ``coverage'' results of BRIA \cite{oesterheldTheoryBoundedInductive2023} and logical induction \cite{garrabrantLogicalInduction2020}, i.e. demonstrating that the market will give a fair chance to the best policy, conditional on some suitable wealth endowments. %
  \item A \textbf{Learning Theory} perspective on markets and the wealth update mechanism. In particular, (real-world) markets appear to have many useful features from an alignment standpoint, such as their inherent capacity for online learning and generalization even from imperfect reward signals.%
  \item A thorough translation of \textbf{economic terminology} into our model -- especially concepts like perfect competition, economies of scale, growth and welfare.%
\end{itemize}

Finally, to accelerate empirical work with market-based algorithms, we plan to release a Python library for efficiently creating and applying market-based algorithms.\todo{not sure how to conclude}

\bibliographystyle{ACM-Reference-Format}
\balance
\bibliography{refs}


\begin{thebibliography}{35}


\ifx \showCODEN    \undefined \def \showCODEN     #1{\unskip}     \fi
\ifx \showDOI      \undefined \def \showDOI       #1{#1}\fi
\ifx \showISBNx    \undefined \def \showISBNx     #1{\unskip}     \fi
\ifx \showISBNxiii \undefined \def \showISBNxiii  #1{\unskip}     \fi
\ifx \showISSN     \undefined \def \showISSN      #1{\unskip}     \fi
\ifx \showLCCN     \undefined \def \showLCCN      #1{\unskip}     \fi
\ifx \shownote     \undefined \def \shownote      #1{#1}          \fi
\ifx \showarticletitle \undefined \def \showarticletitle #1{#1}   \fi
\ifx \showURL      \undefined \def \showURL       {\relax}        \fi
\providecommand\bibfield[2]{#2}
\providecommand\bibinfo[2]{#2}
\providecommand\natexlab[1]{#1}
\providecommand\showeprint[2][]{arXiv:#2}

\bibitem[\protect\citeauthoryear{Arrow and Debreu}{Arrow and Debreu}{1954}]%
        {arrowExistenceEquilibriumCompetitive1954}
\bibfield{author}{\bibinfo{person}{Kenneth~J. Arrow} {and} \bibinfo{person}{Gerard Debreu}.} \bibinfo{year}{1954}\natexlab{}.
\newblock \showarticletitle{Existence of an {{Equilibrium}} for a {{Competitive Economy}}}.
\newblock \bibinfo{journal}{\emph{Econometrica}} \bibinfo{volume}{22}, \bibinfo{number}{3} (\bibinfo{year}{1954}), \bibinfo{pages}{265--290}.
\newblock
\showISSN{0012-9682}
\urldef\tempurl%
\url{https://doi.org/10.2307/1907353}
\showDOI{\tempurl}
\showeprint[jstor]{1907353}


\bibitem[\protect\citeauthoryear{Baum}{Baum}{1999}]%
        {baumModelIntelligenceEconomy1999}
\bibfield{author}{\bibinfo{person}{Eric~B. Baum}.} \bibinfo{year}{1999}\natexlab{}.
\newblock \showarticletitle{Toward a {{Model}} of {{Intelligence}} as an {{Economy}} of {{Agents}}}.
\newblock \bibinfo{journal}{\emph{Machine Learning}} \bibinfo{volume}{35}, \bibinfo{number}{2} (\bibinfo{date}{May} \bibinfo{year}{1999}), \bibinfo{pages}{155--185}.
\newblock
\showISSN{1573-0565}
\urldef\tempurl%
\url{https://doi.org/10.1023/A:1007593124513}
\showDOI{\tempurl}


\bibitem[\protect\citeauthoryear{Chang, Kaushik, Weinberg, Griffiths, and Levine}{Chang et~al\mbox{.}}{2020}]%
        {changDecentralizedReinforcementLearning2020}
\bibfield{author}{\bibinfo{person}{Michael Chang}, \bibinfo{person}{Sid Kaushik}, \bibinfo{person}{S.~Matthew Weinberg}, \bibinfo{person}{Tom Griffiths}, {and} \bibinfo{person}{Sergey Levine}.} \bibinfo{year}{2020}\natexlab{}.
\newblock \showarticletitle{Decentralized {{Reinforcement Learning}}: {{Global Decision-Making}} via {{Local Economic Transactions}}}. In \bibinfo{booktitle}{\emph{Proceedings of the 37th {{International Conference}} on {{Machine Learning}}}}. \bibinfo{publisher}{PMLR}, \bibinfo{pages}{1437--1447}.
\newblock
\showISSN{2640-3498}


\bibitem[\protect\citeauthoryear{Chen, Dai, Du, and Teng}{Chen et~al\mbox{.}}{2009}]%
        {chenSettlingComplexityArrowDebreu2009}
\bibfield{author}{\bibinfo{person}{Xi Chen}, \bibinfo{person}{Decheng Dai}, \bibinfo{person}{Ye Du}, {and} \bibinfo{person}{Shang-Hua Teng}.} \bibinfo{year}{2009}\natexlab{}.
\newblock \showarticletitle{Settling the {{Complexity}} of {{Arrow-Debreu Equilibria}} in {{Markets}} with {{Additively Separable Utilities}}}. In \bibinfo{booktitle}{\emph{2009 50th {{Annual IEEE Symposium}} on {{Foundations}} of {{Computer Science}}}}. \bibinfo{pages}{273--282}.
\newblock
\showISSN{0272-5428}
\urldef\tempurl%
\url{https://doi.org/10.1109/FOCS.2009.29}
\showDOI{\tempurl}


\bibitem[\protect\citeauthoryear{Chen and Teng}{Chen and Teng}{2009}]%
        {chenSpendingNotEasier2009a}
\bibfield{author}{\bibinfo{person}{Xi Chen} {and} \bibinfo{person}{Shang-Hua Teng}.} \bibinfo{year}{2009}\natexlab{}.
\newblock \showarticletitle{Spending {{Is Not Easier Than Trading}}: {{On}} the {{Computational Equivalence}} of {{Fisher}} and {{Arrow-Debreu Equilibria}}}. In \bibinfo{booktitle}{\emph{Algorithms and {{Computation}}}}, \bibfield{editor}{\bibinfo{person}{Yingfei Dong}, \bibinfo{person}{Ding-Zhu Du}, {and} \bibinfo{person}{Oscar Ibarra}} (Eds.). \bibinfo{publisher}{Springer Berlin Heidelberg}, \bibinfo{address}{Berlin, Heidelberg}, \bibinfo{pages}{647--656}.
\newblock
\showISBNx{978-3-642-10631-6}


\bibitem[\protect\citeauthoryear{Conitzer, Freedman, Heitzig, Holliday, Jacobs, Lambert, Mosse, Pacuit, Russell, Schoelkopf, Tewolde, and Zwicker}{Conitzer et~al\mbox{.}}{2024}]%
        {conitzerPositionSocialChoice2024}
\bibfield{author}{\bibinfo{person}{Vincent Conitzer}, \bibinfo{person}{Rachel Freedman}, \bibinfo{person}{Jobst Heitzig}, \bibinfo{person}{Wesley~H. Holliday}, \bibinfo{person}{Bob~M. Jacobs}, \bibinfo{person}{Nathan Lambert}, \bibinfo{person}{Milan Mosse}, \bibinfo{person}{Eric Pacuit}, \bibinfo{person}{Stuart Russell}, \bibinfo{person}{Hailey Schoelkopf}, \bibinfo{person}{Emanuel Tewolde}, {and} \bibinfo{person}{William~S. Zwicker}.} \bibinfo{year}{2024}\natexlab{}.
\newblock \showarticletitle{Position: {{Social Choice Should Guide AI Alignment}} in {{Dealing}} with {{Diverse Human Feedback}}}. In \bibinfo{booktitle}{\emph{Proceedings of the 41st {{International Conference}} on {{Machine Learning}}}}. \bibinfo{publisher}{PMLR}, \bibinfo{pages}{9346--9360}.
\newblock
\showISSN{2640-3498}


\bibitem[\protect\citeauthoryear{Davis}{Davis}{1988}]%
        {davisMappingClassifierSystems1988}
\bibfield{author}{\bibinfo{person}{Lawrence Davis}.} \bibinfo{year}{1988}\natexlab{}.
\newblock \showarticletitle{Mapping Classifier Systems into Neural Networks}. In \bibinfo{booktitle}{\emph{Proceedings of the 1st {{International Conference}} on {{Neural Information Processing Systems}}}} \emph{(\bibinfo{series}{{{NIPS}}'88})}. \bibinfo{publisher}{MIT Press}, \bibinfo{address}{Cambridge, MA, USA}, \bibinfo{pages}{49--56}.
\newblock


\bibitem[\protect\citeauthoryear{Demski}{Demski}{2024}]%
        {demskiCompleteFeedback2024}
\bibfield{author}{\bibinfo{person}{Abram Demski}.} \bibinfo{year}{2024}\natexlab{}.
\newblock \bibinfo{title}{Complete {{Feedback}}}.
\newblock \bibinfo{howpublished}{https://www.alignmentforum.org/posts/3ag99iJEgFFwyj64Z/complete-feedback}.
\newblock


\bibitem[\protect\citeauthoryear{Garrabrant, {Benson-Tilsen}, Critch, Soares, and Taylor}{Garrabrant et~al\mbox{.}}{2020}]%
        {garrabrantLogicalInduction2020}
\bibfield{author}{\bibinfo{person}{Scott Garrabrant}, \bibinfo{person}{Tsvi {Benson-Tilsen}}, \bibinfo{person}{Andrew Critch}, \bibinfo{person}{Nate Soares}, {and} \bibinfo{person}{Jessica Taylor}.} \bibinfo{year}{2020}\natexlab{}.
\newblock \bibinfo{title}{Logical {{Induction}}}.
\newblock
\newblock
\urldef\tempurl%
\url{https://doi.org/10.48550/arXiv.1609.03543}
\showDOI{\tempurl}
\showeprint[arxiv]{1609.03543}~[cs, math]


\bibitem[\protect\citeauthoryear{Ge, Halpern, Micha, Procaccia, Shapira, Vorobeychik, and Wu}{Ge et~al\mbox{.}}{2024}]%
        {geAxiomsAIAlignment2024a}
\bibfield{author}{\bibinfo{person}{Luise Ge}, \bibinfo{person}{Daniel Halpern}, \bibinfo{person}{Evi Micha}, \bibinfo{person}{Ariel~D. Procaccia}, \bibinfo{person}{Itai Shapira}, \bibinfo{person}{Yevgeniy Vorobeychik}, {and} \bibinfo{person}{Junlin Wu}.} \bibinfo{year}{2024}\natexlab{}.
\newblock \showarticletitle{Axioms for {{AI Alignment}} from {{Human Feedback}}}. In \bibinfo{booktitle}{\emph{The {{Thirty-eighth Annual Conference}} on {{Neural Information Processing Systems}}}}.
\newblock


\bibitem[\protect\citeauthoryear{Gode and Sunder}{Gode and Sunder}{1993}]%
        {godeAllocativeEfficiencyMarkets1993}
\bibfield{author}{\bibinfo{person}{Dhananjay~K. Gode} {and} \bibinfo{person}{Shyam Sunder}.} \bibinfo{year}{1993}\natexlab{}.
\newblock \showarticletitle{Allocative {{Efficiency}} of {{Markets}} with {{Zero-Intelligence Traders}}: {{Market}} as a {{Partial Substitute}} for {{Individual Rationality}}}.
\newblock \bibinfo{journal}{\emph{Journal of Political Economy}} \bibinfo{volume}{101}, \bibinfo{number}{1} (\bibinfo{date}{February} \bibinfo{year}{1993}), \bibinfo{pages}{119--137}.
\newblock
\urldef\tempurl%
\url{https://doi.org/10.1086/261868}
\showDOI{\tempurl}


\bibitem[\protect\citeauthoryear{Graves}{Graves}{2017}]%
        {gravesAdaptiveComputationTime2017}
\bibfield{author}{\bibinfo{person}{Alex Graves}.} \bibinfo{year}{2017}\natexlab{}.
\newblock \bibinfo{title}{Adaptive {{Computation Time}} for {{Recurrent Neural Networks}}}.
\newblock
\newblock
\urldef\tempurl%
\url{https://doi.org/10.48550/arXiv.1603.08983}
\showDOI{\tempurl}
\showeprint[arxiv]{1603.08983}~[cs]


\bibitem[\protect\citeauthoryear{Gul and Stacchetti}{Gul and Stacchetti}{1999}]%
        {GUL199995}
\bibfield{author}{\bibinfo{person}{Faruk Gul} {and} \bibinfo{person}{Ennio Stacchetti}.} \bibinfo{year}{1999}\natexlab{}.
\newblock \showarticletitle{Walrasian Equilibrium with Gross Substitutes}.
\newblock \bibinfo{journal}{\emph{Journal of Economic Theory}} \bibinfo{volume}{87}, \bibinfo{number}{1} (\bibinfo{year}{1999}), \bibinfo{pages}{95--124}.
\newblock
\showISSN{0022-0531}
\urldef\tempurl%
\url{https://doi.org/10.1006/jeth.1999.2531}
\showDOI{\tempurl}


\bibitem[\protect\citeauthoryear{Han, Huang, Song, Yang, Wang, and Wang}{Han et~al\mbox{.}}{2021}]%
        {hanDynamicNeuralNetworks2021}
\bibfield{author}{\bibinfo{person}{Yizeng Han}, \bibinfo{person}{Gao Huang}, \bibinfo{person}{Shiji Song}, \bibinfo{person}{Le Yang}, \bibinfo{person}{Honghui Wang}, {and} \bibinfo{person}{Yulin Wang}.} \bibinfo{year}{2021}\natexlab{}.
\newblock \bibinfo{title}{Dynamic {{Neural Networks}}: {{A Survey}}}.
\newblock
\newblock
\urldef\tempurl%
\url{https://doi.org/10.48550/arXiv.2102.04906}
\showDOI{\tempurl}
\showeprint[arxiv]{2102.04906}


\bibitem[\protect\citeauthoryear{Holland}{Holland}{1985}]%
        {hollandPropertiesBucketBrigade1985a}
\bibfield{author}{\bibinfo{person}{John~H. Holland}.} \bibinfo{year}{1985}\natexlab{}.
\newblock \showarticletitle{Properties of the {{Bucket Brigade}}}. In \bibinfo{booktitle}{\emph{Proceedings of the 1st {{International Conference}} on {{Genetic Algorithms}}}}. \bibinfo{publisher}{L. Erlbaum Associates Inc.}, \bibinfo{address}{USA}, \bibinfo{pages}{1--7}.
\newblock
\showISBNx{978-0-8058-0426-3}


\bibitem[\protect\citeauthoryear{Huang and Chang}{Huang and Chang}{2023}]%
        {huangReasoningLargeLanguage2023}
\bibfield{author}{\bibinfo{person}{Jie Huang} {and} \bibinfo{person}{Kevin Chen-Chuan Chang}.} \bibinfo{year}{2023}\natexlab{}.
\newblock \showarticletitle{Towards {{Reasoning}} in {{Large Language Models}}: {{A Survey}}}. In \bibinfo{booktitle}{\emph{Findings of the {{Association}} for {{Computational Linguistics}}: {{ACL}} 2023}}, \bibfield{editor}{\bibinfo{person}{Anna Rogers}, \bibinfo{person}{Jordan {Boyd-Graber}}, {and} \bibinfo{person}{Naoaki Okazaki}} (Eds.). \bibinfo{publisher}{Association for Computational Linguistics}, \bibinfo{address}{Toronto, Canada}, \bibinfo{pages}{1049--1065}.
\newblock
\urldef\tempurl%
\url{https://doi.org/10.18653/v1/2023.findings-acl.67}
\showDOI{\tempurl}


\bibitem[\protect\citeauthoryear{Jamal, Maier, and Sunder}{Jamal et~al\mbox{.}}{2015}]%
        {jamalSimpleAgentsIntelligent2015}
\bibfield{author}{\bibinfo{person}{Karim Jamal}, \bibinfo{person}{Michael~S. Maier}, {and} \bibinfo{person}{Shyam Sunder}.} \bibinfo{year}{2015}\natexlab{}.
\newblock \showarticletitle{Simple {{Agents}}, {{Intelligent Markets}}}.
\newblock \bibinfo{journal}{\emph{SSRN Electronic Journal}} (\bibinfo{year}{2015}).
\newblock
\urldef\tempurl%
\url{https://doi.org/10.2139/ssrn.2478665}
\showDOI{\tempurl}


\bibitem[\protect\citeauthoryear{Kwee, Hutter, and Schmidhuber}{Kwee et~al\mbox{.}}{2001}]%
        {kweeMarketBasedReinforcementLearning2001}
\bibfield{author}{\bibinfo{person}{Ivo Kwee}, \bibinfo{person}{Marcus Hutter}, {and} \bibinfo{person}{Juergen Schmidhuber}.} \bibinfo{year}{2001}\natexlab{}.
\newblock \showarticletitle{Market-{{Based Reinforcement Learning}} in {{Partially Observable Worlds}}}. In \bibinfo{booktitle}{\emph{Proceedings of the {{International Conference}} on {{Artificial Neural Networks}}}}. \bibinfo{publisher}{arXiv}, \bibinfo{pages}{865--873}.
\newblock
\urldef\tempurl%
\url{https://doi.org/10.48550/arXiv.cs/0105025}
\showDOI{\tempurl}
\showeprint[arxiv]{cs/0105025}


\bibitem[\protect\citeauthoryear{{Leonard E Read}}{{Leonard E Read}}{1958}]%
        {leonardereadPencilMyFamily1958}
\bibfield{author}{\bibinfo{person}{{Leonard E Read}}.} \bibinfo{year}{1958}\natexlab{}.
\newblock \showarticletitle{I, {{Pencil}}: {{My Family Tree}} as Told to {{Leonard E}}. {{Read}}}.
\newblock \bibinfo{journal}{\emph{The Freeman}}  \bibinfo{volume}{8} (\bibinfo{date}{December} \bibinfo{year}{1958}), \bibinfo{pages}{32--37}.
\newblock


\bibitem[\protect\citeauthoryear{McKenzie}{McKenzie}{1959}]%
        {mckenzieExistenceGeneralEquilibrium1959}
\bibfield{author}{\bibinfo{person}{Lionel~W. McKenzie}.} \bibinfo{year}{1959}\natexlab{}.
\newblock \showarticletitle{On the {{Existence}} of {{General Equilibrium}} for a {{Competitive Market}}}.
\newblock \bibinfo{journal}{\emph{Econometrica}} \bibinfo{volume}{27}, \bibinfo{number}{1} (\bibinfo{year}{1959}), \bibinfo{pages}{54--71}.
\newblock
\showISSN{0012-9682}
\urldef\tempurl%
\url{https://doi.org/10.2307/1907777}
\showDOI{\tempurl}
\showeprint[jstor]{1907777}


\bibitem[\protect\citeauthoryear{McLau}{McLau}{2024}]%
        {mclauAISearchBitterer2024}
\bibfield{author}{\bibinfo{person}{Aidan McLau}.} \bibinfo{year}{2024}\natexlab{}.
\newblock \bibinfo{title}{{{AI Search}}: {{The Bitter-er Lesson}}}.
\newblock \bibinfo{howpublished}{https://news.ycombinator.com/item?id=40683697}.
\newblock


\bibitem[\protect\citeauthoryear{Miller}{Miller}{2006}]%
        {millerNotesMicroeconomicTheory2006}
\bibfield{author}{\bibinfo{person}{Nolan Miller}.} \bibinfo{year}{2006}\natexlab{}.
\newblock \bibinfo{title}{Notes on {{Microeconomic Theory}}}.  (\bibinfo{date}{August} \bibinfo{year}{2006}).
\newblock
\newblock
\shownote{(Lecture Notes).}


\bibitem[\protect\citeauthoryear{Minsky}{Minsky}{1988}]%
        {minskySocietyMind1988}
\bibfield{author}{\bibinfo{person}{Marvin Minsky}.} \bibinfo{year}{1988}\natexlab{}.
\newblock \bibinfo{booktitle}{\emph{Society {{Of Mind}}}}.
\newblock \bibinfo{publisher}{{Simon and Schuster}}.
\newblock
\showISBNx{978-0-671-65713-0}


\bibitem[\protect\citeauthoryear{Neyman}{Neyman}{2024}]%
        {neymanAlgorithmicBayesianEpistemology2024}
\bibfield{author}{\bibinfo{person}{Eric Neyman}.} \bibinfo{year}{2024}\natexlab{}.
\newblock \bibinfo{title}{Algorithmic {{Bayesian Epistemology}}}.
\newblock
\newblock
\urldef\tempurl%
\url{https://doi.org/10.48550/arXiv.2403.07949}
\showDOI{\tempurl}
\showeprint[arxiv]{2403.07949}~[cs]


\bibitem[\protect\citeauthoryear{Oesterheld, Demski, and Conitzer}{Oesterheld et~al\mbox{.}}{2023}]%
        {oesterheldTheoryBoundedInductive2023}
\bibfield{author}{\bibinfo{person}{Caspar Oesterheld}, \bibinfo{person}{Abram Demski}, {and} \bibinfo{person}{Vincent Conitzer}.} \bibinfo{year}{2023}\natexlab{}.
\newblock \showarticletitle{A {{Theory}} of {{Bounded Inductive Rationality}}}.
\newblock \bibinfo{journal}{\emph{Electronic Proceedings in Theoretical Computer Science}}  \bibinfo{volume}{379} (\bibinfo{date}{July} \bibinfo{year}{2023}), \bibinfo{pages}{421--440}.
\newblock
\showISSN{2075-2180}
\urldef\tempurl%
\url{https://doi.org/10.4204/EPTCS.379.33}
\showDOI{\tempurl}
\showeprint[arxiv]{2307.05068}~[cs]


\bibitem[\protect\citeauthoryear{{OpenAI}}{{OpenAI}}{2024}]%
        {openaiLearningReasonLLMs2024}
\bibfield{author}{\bibinfo{person}{{OpenAI}}.} \bibinfo{year}{2024}\natexlab{}.
\newblock \bibinfo{title}{Learning to {{Reason}} with {{LLMs}}}.
\newblock
\newblock


\bibitem[\protect\citeauthoryear{Rahaman, Weiss, W{\"u}thrich, Bengio, Li, Pal, and Sch{\"o}lkopf}{Rahaman et~al\mbox{.}}{2024}]%
        {rahamanLanguageModelsCan2024}
\bibfield{author}{\bibinfo{person}{Nasim Rahaman}, \bibinfo{person}{Martin Weiss}, \bibinfo{person}{Manuel W{\"u}thrich}, \bibinfo{person}{Yoshua Bengio}, \bibinfo{person}{Li~Erran Li}, \bibinfo{person}{Chris Pal}, {and} \bibinfo{person}{Bernhard Sch{\"o}lkopf}.} \bibinfo{year}{2024}\natexlab{}.
\newblock \bibinfo{title}{Language {{Models Can Reduce Asymmetry}} in {{Information Markets}}}.
\newblock
\newblock
\urldef\tempurl%
\url{https://doi.org/10.48550/arXiv.2403.14443}
\showDOI{\tempurl}
\showeprint[arxiv]{2403.14443}~[cs]


\bibitem[\protect\citeauthoryear{Rao, Steeves, Shaabana, Attevelt, and McAteer}{Rao et~al\mbox{.}}{2021}]%
        {raoBitTensorPeertoPeerIntelligence2021}
\bibfield{author}{\bibinfo{person}{Yuma Rao}, \bibinfo{person}{Jacob Steeves}, \bibinfo{person}{Ala Shaabana}, \bibinfo{person}{Daniel Attevelt}, {and} \bibinfo{person}{Matthew McAteer}.} \bibinfo{year}{2021}\natexlab{}.
\newblock \bibinfo{title}{{{BitTensor}}: {{A Peer-to-Peer Intelligence Market}}}.
\newblock
\newblock
\urldef\tempurl%
\url{https://doi.org/10.48550/arXiv.2003.03917}
\showDOI{\tempurl}
\showeprint[arxiv]{2003.03917}


\bibitem[\protect\citeauthoryear{Schmidhuber}{Schmidhuber}{1987}]%
        {schmidhuberEvolutionaryPrinciplesSelfreferential1987}
\bibfield{author}{\bibinfo{person}{J{\"u}rgen Schmidhuber}.} \bibinfo{year}{1987}\natexlab{}.
\newblock \showarticletitle{Evolutionary Principles in Self-Referential Learning, or on Learning How to Learn: {{The}} Meta-Meta-. Hook}.
\newblock


\bibitem[\protect\citeauthoryear{Schmidhuber}{Schmidhuber}{1989}]%
        {schmidhuberLocalLearningAlgorithm1989}
\bibfield{author}{\bibinfo{person}{Jurgen Schmidhuber}.} \bibinfo{year}{1989}\natexlab{}.
\newblock \showarticletitle{A {{Local Learning Algorithm}} for {{Dynamic Feedforward}} and {{Recurrent Networks}}}.
\newblock \bibinfo{journal}{\emph{Connection Science}} \bibinfo{volume}{1}, \bibinfo{number}{4} (\bibinfo{date}{January} \bibinfo{year}{1989}), \bibinfo{pages}{403--412}.
\newblock
\showISSN{0954-0091}
\urldef\tempurl%
\url{https://doi.org/10.1080/09540098908915650}
\showDOI{\tempurl}


\bibitem[\protect\citeauthoryear{Schwartz}{Schwartz}{2008}]%
        {schwartzHowMuchIrrationality2008}
\bibfield{author}{\bibinfo{person}{Alan Schwartz}.} \bibinfo{year}{2008}\natexlab{}.
\newblock \showarticletitle{How {{Much Irrationality Does}} the {{Market Permit}}?}
\newblock \bibinfo{journal}{\emph{The Journal of Legal Studies}} \bibinfo{volume}{37}, \bibinfo{number}{1} (\bibinfo{date}{January} \bibinfo{year}{2008}), \bibinfo{pages}{131--159}.
\newblock
\urldef\tempurl%
\url{https://doi.org/10.1086/519963}
\showDOI{\tempurl}


\bibitem[\protect\citeauthoryear{von Hayek}{von Hayek}{1937}]%
        {hayekEconomicsKnowledge1937}
\bibfield{author}{\bibinfo{person}{F.~A. von Hayek}.} \bibinfo{year}{1937}\natexlab{}.
\newblock \showarticletitle{Economics and {{Knowledge}}}.
\newblock \bibinfo{journal}{\emph{Economica}} \bibinfo{volume}{4}, \bibinfo{number}{13} (\bibinfo{year}{1937}), \bibinfo{pages}{33--54}.
\newblock
\showISSN{00130427, 14680335}
\showeprint[jstor]{2548786}


\bibitem[\protect\citeauthoryear{Wentworth}{Wentworth}{2018}]%
        {wentworthCompetitiveMarketsDistributed2018}
\bibfield{author}{\bibinfo{person}{John Wentworth}.} \bibinfo{year}{2018}\natexlab{}.
\newblock \bibinfo{title}{Competitive {{Markets}} as {{Distributed Backprop}}}.
\newblock
\newblock


\bibitem[\protect\citeauthoryear{{Yann LeCun}}{{Yann LeCun}}{2023}]%
        {yannlecunMachinesThatCan2023}
\bibfield{author}{\bibinfo{person}{{Yann LeCun}}.} \bibinfo{year}{2023}\natexlab{}.
\newblock \showarticletitle{Towards {{Machines}} That Can {{Learn}}, {{Reason}}, and {{Plan}}}. In \bibinfo{booktitle}{\emph{{{AI}} and Barrier of Meaning {{Workshop}}}}. \bibinfo{address}{Santa Fe Institute}.
\newblock


\bibitem[\protect\citeauthoryear{Yao, Yu, Zhao, Shafran, Griffiths, Cao, and Narasimhan}{Yao et~al\mbox{.}}{2024}]%
        {yaoTreeThoughtsDeliberate2024}
\bibfield{author}{\bibinfo{person}{Shunyu Yao}, \bibinfo{person}{Dian Yu}, \bibinfo{person}{Jeffrey Zhao}, \bibinfo{person}{Izhak Shafran}, \bibinfo{person}{Thomas~L. Griffiths}, \bibinfo{person}{Yuan Cao}, {and} \bibinfo{person}{Karthik Narasimhan}.} \bibinfo{year}{2024}\natexlab{}.
\newblock \showarticletitle{Tree of Thoughts: Deliberate Problem Solving with Large Language Models}. In \bibinfo{booktitle}{\emph{Proceedings of the 37th {{International Conference}} on {{Neural Information Processing Systems}}}} \emph{(\bibinfo{series}{{{NIPS}} '23})}. \bibinfo{publisher}{Curran Associates Inc.}, \bibinfo{address}{Red Hook, NY, USA}, \bibinfo{pages}{11809--11822}.
\newblock


\end{thebibliography}

\end{document}